\documentclass[twoside]{article}
\usepackage{aistats2018arxiv}

\usepackage{algorithm,algorithmic}
\usepackage{amsmath}
\usepackage{natbib}
\usepackage{hhline}
\usepackage[caption=false]{subfig}

\usepackage{hyperref}
\usepackage{graphics}
\usepackage{times}
\usepackage{epstopdf}
\usepackage{xcolor}
\usepackage{graphicx} 

\newtheorem{theorem}{Theorem}
\newtheorem{lemma}[theorem]{Lemma}

\ifx\assumption\undefined
\newtheorem{assumption}{Assumption}
\fi

\ifx\proposition\undefined
\newtheorem{proposition}[theorem]{Proposition}
\fi

\ifx\remark\undefined
\newtheorem{remark}{Remark}
\fi

\ifx\definition\undefined
\newtheorem{definition}{Definition}
\fi

\ifx\proof\undefined
\newenvironment{proof}{\par\noindent{\bf Proof\ }}{\hfill\BlackBox\\[2mm]}
\fi

\input{Definitions}

\newcommand{\RN}[1]{%
	\textup{\lowercase\expandafter{\it \romannumeral#1}}%
}

\begin{document} 

\twocolumn[

\aistatstitle{On Connecting Stochastic Gradient MCMC and Differential Privacy}

\aistatsauthor{ Bai Li \And Changyou Chen \And  Hao Liu \And Lawrence Carin }

\aistatsaddress{Duke University \And  University at Buffalo, SUNY \And Nanjing University \And Duke University } ]

\begin{abstract}
Significant success has been realized recently on applying machine learning to real-world applications. There have also been corresponding concerns on the privacy of training data, which relates to data security and confidentiality issues. Differential privacy provides a principled and rigorous privacy guarantee on machine learning models. While it is common to design a model satisfying a required differential-privacy property by injecting noise, it is generally hard to balance the trade-off between privacy and utility. We show that stochastic gradient Markov chain Monte Carlo (SG-MCMC) -- a class of scalable Bayesian posterior sampling algorithms proposed recently -- satisfies strong differential privacy with carefully chosen step sizes. We develop theory on the performance of the proposed differentially-private SG-MCMC method. We conduct experiments to support our analysis, and show that a standard SG-MCMC sampler without any modification (under a default setting) can reach state-of-the-art performance in terms of both privacy and utility on Bayesian learning.
\end{abstract}

	\section{Introduction}
	\label{sec:intro}

Utilizing large amounts of data has helped machine learning algorithms achieve significant success in many real applications. However, such work also raises privacy concerns. For example, a diagnostic system based on machine learning algorithms may be trained on a large quantity of patient data, such as medical images. It is important to protect training data from adversarial attackers \citep{shokri2017membership}. However, even the most widely-used machine learning algorithms such as deep learning could implicitly memorize the training data \citep{papernot2016semi}, meaning that the learned model parameters implicitly contain information that could violate the privacy of training data. Such algorithms may be readily attacked.

The above potential model vulnerability can be addressed by differential privacy (DP), a general notion of algorithm privacy \citep{dwork2008differential,dwork2006calibrating}. It is designed to provide a strong privacy guarantee for general learning procedures, such as statistical analysis and machine learning algorithms, that involve private information.
	
	Among the popular machine learning algorithms, Bayesian inference has realized significant success recently, due to its capacity to leverage expert knowledge and employ uncertainty estimates. Notably, the recently developed stochastic gradient Markov chain Monte Carlo (SG-MCMC) technique enables scalable Bayesian inference in a big-data setting. While there have been many extensions of SG-MCMC, little work has been directed at studying the privacy properties of such algorithms. Specifically, \citet{wang2015privacy} showed that an SG-MCMC algorithm with appropriately chosen step sizes preserves differential privacy. In practice, however, their analysis requires the step size to be extremely small to limit the risk of violating privacy. Such a small step size is not practically useful for training models with non-convex posterior distribution landscapes, which is the most common case in recent machine learning models. More details of this issue are discussed in Section~\ref{sec:step size_bound}.
	
	On the other hand, \citet{abadi2016deep} introduced a new privacy-accounting method, which allows one to keep better track of the privacy loss (defined in Section~\ref{sec:dp_definition}) for sequential algorithms. Further, they proposed a differentially-private stochastic gradient descending (DP-SGD) method for training machine learning models privately. Although they showed a significant improvement in calculating the privacy loss, there is no theory showing that their DP-SGD has a guaranteed performance under privacy constraints.

	In this paper, built on the notation of the privacy accounting method, we show that using SG-MCMC for training large-scale machine learning models is sufficient to achieve strong differential privacy. Specifically, we combine the advantages of the aforementioned works, and prove that SG-MCMC methods naturally satisfy the definition of differential privacy even without changing their default step size, thus allowing both good utility and strong privacy in practice.


	\section{Preliminaries}
	\label{sec:review}
	The following notation is used through out the paper. An input database containing $N$ data points is represented as $X=(\db_1,\dots,\db_N)\in \mathcal{X}^N$, where $\db_i\in\mathcal{X}$. The parameters of interest in the model are denoted $\thetab \in \mathbb{R}^r$, {\it e.g.}, these may be the weights of a deep neural network. The identity matrix is denoted $\Ib$.
	
	\subsection{Differential Privacy}\label{sec:dp_definition}

	The concept of differential privacy was proposed by \citet{dwork2008differential} to describe the privacy modeling property of a randomized mechanism (algorithm) on two adjacent datasets.

	\begin{definition}[Adjacent Datasets]
		Two datasets $X$ and $X^\prime$ are called adjacent if they only differ by one record, {\it e.g.}, $\mathbf{d}_i \neq \mathbf{d}_i^\prime$ for some $i$, where $\mathbf{d}_i \in X$ and $\mathbf{d}_i^\prime \in X^\prime$. 
	\end{definition}

	\begin{definition}[Differential Privacy]\label{def:dp}
	Given a pair of adjacent datasets $X$ and $X^\prime$, a randomized mechanism $\mathcal{M}:\mathcal{X}^N\rightarrow\mathcal{Y}$ mapping from data space to its range $\mathcal{Y}$ satisfies $(\epsilon,\delta)$-differential privacy if for all measurable $\mathcal{S} \subset \text{range}(\mathcal{M})$ and all adjacent $X$ and $ X^\prime$
		$$
		Pr(\mathcal{M}(X) \in \mathcal{S }) \leq e^{\epsilon} Pr(\mathcal{M}(X^\prime) \in
		\mathcal{S}) + \delta
		$$
	where $Pr(e)$ denotes the probability of event $e$, and $\epsilon$ and $\delta$ are two positive real numbers that indicate the loss of privacy. When $\delta=0$, we say the mechanism $\mathcal{M}$ has $\epsilon$-differential privacy.
\end{definition}

	Differential privacy places constraints on the difference between the outputs of two adjacent inputs $X$ and $X^\prime$ by a random mechanism. If we assume that $X$ and $X^\prime$ only differ by one record $\db_i$, by observing the outputs, any outside attackers are not able to recognize whether the output has resulted from $X$ and $X^\prime$, as long as $\epsilon$ and $\delta$ are small enough (making these two probabilities close to each other). Thus, the existence of the record $\db_i$ is protected. Since the record in which the two datasets differ by is arbitrary, the privacy protection is applicable for all records.

	To better describe the randomness of the outputs of $\mathcal{M}$ with inputs $X$ and $X^\prime$, we define a random variable called privacy loss.

	\begin{definition}[Privacy Loss]
		Let $\mathcal{M}:\mathcal{X}^N\rightarrow\mathcal{Y}$ be a randomized mechanism, and $X$ and $X^\prime$ are a pair of adjacent datasets. Let $\mathsf{aux}$ denote any auxiliary input that does not depend on $X$ or $X^\prime$. For an outcome $o\in\mathcal{Y}$ from the mechanism $\mathcal{M}$, the privacy loss at $o$ is defined as:
		$$c(o;\mathcal{M},\textsf{aux},X,X^\prime)\overset{\Delta}{=}\log\frac{Pr[\mathcal{M}(\textsf{aux},X)=o]}{Pr[\mathcal{M}(\textsf{aux},X^\prime)=o]}$$
	\end{definition}

	It can be shown that the $(\epsilon,\delta)$-DP is equivalent to the tail bound of the distribution of its corresponding privacy loss random variable \citep{abadi2016deep} (see Theorem~\ref{theorem:1} in the next section), thus this random variable is an important tool for quantifying the privacy loss of a mechanism.

	\subsection{Moments Accountant Method}
	\label{sec:moment}
	A common approach for achieving differential privacy is to introduce random noise, to hide the existence of a particular data point. For example, Laplace and Gaussian mechanisms \citep{dwork2014algorithmic} add {\it i.i.d.}\! Laplace random noise and Gaussian noise, respectively, to an algorithm. While a large amount of noise makes an algorithm differentially private, it may sacrifice the utility of the algorithm. Therefore, in such paradigms, it is important to calculate the smallest amount of noise that is required to achieve a certain level of differential privacy.

	The moments accountant method proposed in \citet{abadi2016deep} keeps track of a bound of the moments of the privacy loss random variables defined above. As a result, it allows one to calculate the amount of noise needed to ensure the privacy loss under a given threshold.

	\begin{definition}[Moments Accountant]
		Let $\mathcal{M}:\mathcal{X}^N\rightarrow\mathcal{Y}$ be a randomized mechanism, and let $X$ and $X^\prime$ be a pair of adjacent data sets. Let $\mathsf{aux}$ denote any auxiliary input that is independent of both $X$ and $X^\prime$. The moments accountant with an integer parameter $\lambda$ is defined as:
		$$\alpha_{\mathcal{M}}(\lambda)\overset{\Delta}{=}\underset{\textsf{aux},d,d^\prime}{\max}\alpha_{\mathcal{M}}(\lambda;\textsf{aux},X,X^\prime) $$
	 where $\alpha_\mathcal{M}(\lambda;\textsf{aux},X,X^\prime)\triangleq\log\mathbb{E}[\exp(\lambda c(\mathcal{M},\textsf{aux},X,X^\prime))]$ is the log of the moment generating function evaluated at $\lambda$, that is the $\lambda^{\text{th}}$ moment of the privacy loss random variable.
	\end{definition}

	The following moments bound on Gaussian mechanism with random sampling is proved in \citep{abadi2016deep}.

\begin{theorem}
	\label{theorem:1}
	\textbf{[Composability]} Suppose that $\mathcal{M}$ consists of a sequence of adaptive mechanisms $\mathcal{M}_1,\dots,\mathcal{M}_k$ where $\mathcal{M}_i:\prod_{j=1}^{i-1}\mathcal{Y}_j\times \mathcal{X}\rightarrow\mathcal{Y}_i$, and $\mathcal{Y}_i$ is the range of the $i$th mechanism, {\it i.e.}, $\mathcal{M} = \mathcal{M}_k\circ \cdots\circ\mathcal{M}_1$, with $\circ$ the composition operator. Then, for any $\lambda$
	$$\alpha_{\mathcal{M}}(\lambda)=\sum_{i=1}^k\alpha_{\mathcal{M}_i}(\lambda)$$
	where the auxiliary input for $\alpha_{\mathcal{M}_i}$ is defined as all $\alpha_{\mathcal{M}_j}$'s outputs, $\{o_j\}$, for $j < i$; and $\alpha_{\mathcal{M}}$ takes $\mathcal{M}_i'$s output, $\{o_i\}$ for $i<k$, as the auxiliary input.

	\textbf{[Tail bound]} For any $\epsilon>0$, the mechanism $\mathcal{M}$ is $(\epsilon,\delta)$-DP for
	$$\delta=\underset{\lambda}{\min}\exp\left(\alpha_\mathcal{M}(\lambda)-\lambda\epsilon\right)$$

\end{theorem}

	For the rest of this paper, for simplicity we only consider mechanisms that output a real-valued vector. That is, $\mathcal{M}:\mathcal{X}^N\rightarrow \mathbb{R}^p$.
	
	Using the properties above, the following lemma about the moments accountant has been proven in \citep{abadi2016deep}:

		

\begin{lemma}
	\label{lemma:1}
	Suppose that $f:\mathcal{D}\rightarrow \mathbb{R}^p$ with $\|f(.)\|\leq 1$. Let $\sigma\geq 1$ and $J$ is a mini-batch sample with sampling probability $q$, {\it i.e.}, $q=\frac{\tau}{N}$ with minibatch size of $\tau$. If $q<\frac{1}{16\sigma}$, for any positive integer $\lambda\leq \sigma^2\ln \frac{1}{q\sigma}$, the mechanism $\mathcal{M}(X) = \sum_{i\in J} f(\db_i)+N(0,\sigma^2I)$ satisfies
		$$\alpha_{\mathcal{M}}(\lambda)\leq \frac{q^2\lambda(\lambda+1)}{(1-q)\sigma^2}+O(q^3\lambda^3/\sigma^3)$$
		 
\end{lemma}

	In the following, we build our analysis of the differentially-private SG-MCMC based on this lemma.
	
	\subsection{Stochastic Gradient Markov Chain Monte Carlo}
	SG-MCMC is a family of scalable Bayesian sampling algorithms, developed recently to generate approximate samples from a posterior distribution $p(\thetab | X)$, with $\thetab$ a model parameter vector. SG-MCMC mitigates the slow mixing and non-scalability issues encountered by traditional MCMC algorithms, by $\RN{1})$ adopting gradient information of the posterior distribution, and $\RN{2})$ using minibatches of data in each iteration of the algorithm. It is particularly suitable for large-scale Bayesian learning, and thus is becoming increasingly popular.
	
	SG-MCMC algorithms are discretized numerical approximations of continuous-time It\^{o} diffusions \citep{ChenDC:NIPS15,MaCF:NIPS15}, whose stationary distributions are designed to coincide with $p(\thetab | X)$. Formally, an It\^{o} diffusion is written as
	\begin{align}\label{eq:ito}
	\mathrm{d}\Thetab_t &= F(\Thetab_t)\mathrm{d}t + g(\Thetab_t)\mathrm{d}\mathcal{W}_t~,
	\end{align}
	with $t$ is the time index; $\Thetab_t \in \mathbb{R}^p$ represents the full variables in a system, where typically $\Thetab_t \supseteq \thetab_t$ (thus $p \geq r$) is an augmentation of the model parameters; and $\mathcal{W}_t \in \mathbb{R}^p$ is $p$-dimensional Brownian motion. Functions $F: \mathbb{R}^p \to \mathbb{R}^p$ and $g: \mathbb{R}^p \rightarrow \mathbb{R}^{p\times p}$ are assumed to satisfy the Lipschitz continuity condition \citep{Ghosh:book11}.
	
	Based on the It\^{o} diffusion, SG-MCMC algorithms further develop three components for scalable inference: $\RN{1})$ define appropriate functions $F$ and $g$ in \eqref{eq:ito} so that their (marginal) stationary distributions coincide with the target posterior distribution $p(\thetab|X)$; $\RN{2})$ replace $F$ or $g$ with unbiased stochastic approximations to reduce the computational complexity, {\it e.g.}, approximating $F$ with a random subset of data points instead of using the full data; and $\RN{3})$ solve the generally intractable continuous-time It\^{o} diffusions with a numerical method, which typically brings estimation errors that are controllable.  
	
	The stochastic gradient Langevin dynamic (SGLD) model defines $\Thetab = \thetab$, and $F(\Thetab_t) = -\nabla_{\thetab} U(\thetab), \hspace{0.1cm} g(\Thetab_t) = \sqrt{2}\Ib_r$, where $U(\thetab) \triangleq -\log p(\thetab) - \sum_{i=1}^N \log p(\db_i | \thetab)$ denotes the unnormalized negative log-posterior, and $p(\thetab)$ is the prior distribution of $\thetab$. 
	The stochastic gradient Hamiltonian Monte Carlo (SGHMC) method \citep{pmlr-v32-cheni14} is based on second-order Langevin dynamics, 
	which defines $\Thetab = (\thetab, \qb)$, and 
	{\small\begin{align*}
	F(\Thetab_t)= \Big( \begin{array}{c}
	\qb \\
	-B \qb-\nabla_\thetab U(\thetab) \end{array} \Big),\hspace{0.1cm}
	g(\Thetab_t) = \sqrt{2B}\Big( \begin{array}{cc}
	{\bf 0} & {\bf 0} \\
	{\bf 0} & \Ib_n \end{array} \Big)
	\end{align*}}
	for a scalar $B > 0$; $\qb$ is an auxiliary variable known as the momentum \citep{pmlr-v32-cheni14,DingFBCSN:NIPS14}.
	Similar formulae can be defined for other SG-MCMC algorithms, such as the stochastic gradient 
	thermostat \citep{DingFBCSN:NIPS14}, and other variants with Riemannian information 
	geometry \citep{PattersonT:NIPS13,MaCF:NIPS15,LICCC:AAAI16}.

	To make the algorithms scalable in a large-data setting, {\it i.e.}, when $N$ is large, an unbiased version of $\nabla_\thetab U(\thetab)$ is calculated with a random subset of the full data, denoted $\nabla_\thetab \tilde{U}(\thetab)$ and defined as
	\begin{align*}
		\nabla_\thetab \tilde{U}(\thetab) = \nabla\log p(\thetab)+ \frac{N}{\tau}\sum_{\db_i\in J}\log p(\db_i|\thetab)~,
	\end{align*} 
	where $J$ is a random minibatch of the data with size $\tau$ (typically $\tau \ll N$).
	
	We typically adopt the popular Euler method to solve the continuous-time diffusion by an $\eta$-time discretization (step size being $\eta$). The Euler method is a first-order numerical integrator, thus inducing an $O(\eta)$ approximation error \citep{ChenDC:NIPS15}. Algorithm \ref{alg:1} illustrates the application of SGLD algorithm with the Euler integrator for differential privacy, which is almost the same as original SGLD except that there is a gradient norm clipping in Step 4 of the algorithm. The norm-clipping step ensures that the computed gradients satisfy the Lipschitz condition, a common assumption on loss functions in a differential-privacy setting \citep{song2013stochastic,bassily2014differentially,wang2015privacy}. The reasoning is intuitively clear: since differential privacy requires the output to be non-sensitive to any changes on an arbitrary data point, it is thus crucial to bound the impact of a single data point to the target function. The Lipschitz condition is easily met by clipping the norm of a loss function, a common technique for gradient-based algorithms to prevent gradient explosion \citep{pascanu2013difficulty}.
	
	The clipping is equivalent to using an adaptive step size as in preconditioned SGLD \citep{LICCC:AAAI16}, and thus it does not impact its convergence rate in terms of the estimation accuracy discussed in Section~\ref{sec:bound}.

	\begin{algorithm}
		\caption{Stochastic Gradient Langevin Dynamics with Differential Privacy}
		\begin{algorithmic}[1]
			\label{alg:1}
			\REQUIRE Data $X$ of size $N$, size of mini-batch $\tau$, number of iterations $T$, prior $p(\boldsymbol{\thetab})$, privacy parameter $\epsilon, \delta$, gradient norm bound $L$. A decreasing/fixed-step-size sequence $\{\eta_t\}$. Set $t=1$.
			\FOR{$t\in[T]$}
			\STATE Take a random sample $J_t$ with sampling probability $q = \tau/N$.
			\STATE Calculate $g_t(\db_i)\leftarrow \nabla \log\ell(\boldsymbol{\thetab}_t|\db_i)$
			\STATE Clip norm: $\tilde{g}_t(\db_i)\leftarrow g_t(\db_i)/\max\left(1,\frac{\|g_t(\db_i)\|_2}{L}\right)$
			\STATE Sample each coordinate of $\mathbf{z}_t$ iid from $N(0, \frac{\eta_t}{N}I)$
			\STATE Update $\boldsymbol{\thetab}_{t+1} \leftarrow\boldsymbol{\thetab}_t-\eta_t\left(\frac{\nabla\log p(\boldsymbol{\thetab})}{N}+\frac{1}{\tau}\sum_{i\in J_t}\tilde{g}_t(\db_i)\right) +\mathbf{z}_t$ 
			\STATE Return $\boldsymbol{\thetab}_{t+1}$ as a posterior sample (after a predefined burn-in period).
			\STATE Increment $t\leftarrow t+1$.
			\ENDFOR
			\STATE $\thetab_T$ and compute the overall privacy cost $(\epsilon,\delta)$ using a the moment accountant method.
		\end{algorithmic}
	\end{algorithm}

	\section{Privacy Analysis for Stochastic Gradient Langevin Dynamics}
	\label{sec:dp}
	We first develop theory to prove Algorithm~\ref{alg:1} is $(\epsilon,\delta)$-DP under a certain condition. Our theory shows a significant improvement of the differential privacy obtained by SGLD over the most related work by \cite{wang2015privacy}. To study the estimation accuracy (utility) of the algorithm, the corresponding mean square error estimation bounds are then proved under such differential-privacy settings.
	
	\subsection{Step size bounds for differentially-private SGLD}\label{sec:step size_bound}
	Previous work on SG-MCMC has shown that an appropriately chosen decreasing step size sequence can be adopted for an SG-MCMC algorithm \citep{TehTV:arxiv14,ChenDC:NIPS15}. For the sequence in the form of $\eta_t = O(t^{-\alpha})$, the optimal value is $\alpha = \frac{1}{3}$ in order to obtain the optimal mean square error bound (defined in Section~\ref{sec:bound}). Consequently, we first consider $\eta_t = O(t^{-1/3})$ in our analysis below, where the constant of the stepsize can be specified with parameters of the DP setting, shown in Theorem~\ref{theo:dp}. The differential privacy property under a fixed step size is also discussed subsequently.

	\begin{theorem}\label{theo:dp}
	If we let the step size decrease at the rate of $O(t^{-1/3})$, there exist positive constants $c_1$ and $c_2$ such that given the sampling probability $q=\tau/N$ and the number of iterations $T$, for any $\epsilon<c_1q^2T^{2/3}$, Algorithm \ref{alg:1} satisfies $(\epsilon, \delta)$-DP as long as $\eta_t$ satisfies:
	\begin{enumerate}
		\item $\eta_t\leq \frac{N}{L^2}$ 
		\item $\eta_t>\frac{q^2N}{256L^2}$
		\item $\eta_t<\frac{\epsilon^2Nt^{-1/3}}{c_2^2L^2T^{2/3}\log(1/\delta)}$.
	\end{enumerate}
	\end{theorem}

	\begin{proof}
	See Section \ref{app:theorem3} of the SM.
	\end{proof}

\begin{remark} In practice, the first condition is easy to satisfy as $\frac{N}{L^2}$ is often much larger than the step size. The second condition is also easy to satisfy with properly chosen $L$ and $q$, and we will verify this condition in our experiments. In the rest of this section, we only focus on the third condition as an upper bound to the step size.\end{remark}

It is now clear that with the optimal decreasing step size sequence (in terms of MSE defined in Section~\ref{sec:bound}), Algorithm~\ref{alg:1} maintains $(\epsilon, \delta)$-DP. There are other variants of SG-MCMC which use fixed step sizes. We show in Theorem~\ref{remark:fix_DP} that in this case, the algorithm still satisfies $(\epsilon, \delta)$-DP.
	\begin{theorem}\label{remark:fix_DP}
		Under the same setting as Theorem~\ref{theo:dp}, but using a fixed-step size $\eta_t = \eta$, Algorithm \ref{alg:1} satisfies $(\epsilon,\delta)$-DP whenever $\eta<\frac{\epsilon^2N}{c^2L^2Tlog(1/\delta)}$ for another constant $c$.
	\end{theorem}
	\begin{proof}
		See Section~\ref{app:fixed_DP} of the SM.
	\end{proof}
In \citep{wang2015privacy}, the authors proved that the SGLD method is $(\epsilon, \delta)$-DP if the step size $\eta_t$ is small enough to satisfy \begin{align*}
	\eta_t<&\frac{\epsilon^2N}{128L^2T\log\left(\frac{2.5T}{\delta}\right)\log(2/\delta)}~.
\end{align*}

This bound is relatively small compared to ours (explained below), thus it is not practical in real applications. To address this problem, \cite{wang2015privacy} proposed the Hybrid Posterior Sampling algorithm, that uses the One Posterior Sample (OPS) estimator for the ``burn-in'' period, followed by the SGLD with a small step size to guarantee the differential privacy property.
We note that for complicated models, especially with non-convex target posterior landscapes, such an upper bound for step size still brings practical problems even with the OPS. One issue is that the Markov chain will mix very slowly with a small step size, leading to highly correlated samples.

By contrast, our new upper bound for the step size in Theorem~\ref{theo:dp}, $\eta_t<\frac{\epsilon^2Nt^{-1/3}}{c_2^2L^2T^{2/3}\log(1/\delta)}$, improves the bound in \cite{wang2015privacy} by a factor of $T^{1/3}\log (T/\delta)$ at the first iteration. Note the constant $c_2^2$ in our bound is empirically smaller than $128$ (see the calculating method in Section~\ref{app:constants} of the SM), thus still giving a larger bound overall.

To provide intuition on how our bound compares with that in \cite{wang2015privacy}, consider the MNIST data set with $N=50,000$. If we set $\epsilon=0.1$, $\delta=10^{-5}$,
$T=10000$, and $L=1$, our upper bound can be calculated as $\eta_t<0.103$, consistent with the default step size when training MNIST \citep{LICCC:AAAI16}. More importantly, our theory indicates that using SGLD with the default step size $\eta_t=0.1$ is able to achieve $(\epsilon,\delta)$-DP with a small privacy loss for the MNIST dataset. As a comparison, \citep{wang2015privacy} gives a much smaller upper bound of $\eta_t<1.54\times 10^{-6}$, which is too small too be practically used. More detailed comparison for these two bounds is given in Section~\ref{sec:exp_upper}, when considering experimental results. 

Finally, note that as in \citep{wang2015privacy}, our analysis can be easily extended to other SG-MCMC methods such as SGHMC \citep{pmlr-v32-cheni14} and SGNHT \citep{DingFBCSN:NIPS14}. We do not specify the results here for conciseness.

\subsection{Utility Bounds}
\label{sec:bound}

The above theory indicates that, with a smaller step size, one can manifest an SG-MCMC algorithm that preserves more privacy, {\it e.g.}, $(0, \delta)$-DP in the limit of zero step size. On the other hand, when the step size approaches zero, we get (theoretically) exact samples from the posterior distributions. In this case, the implication of privacy becomes transparent because changing one data point typically would not impact prediction under a posterior distribution in a Bayesian model. However, as we note above, this does not mean we can choose arbitrarily small step sizes, because this would hinder the exploration of the parameter space, leading to slow mixing.

To measure the mixing and utility property, we investigate the estimation accuracy bounds under the differential privacy setting. Following standard settings for SG-MCMC \citep{ChenDC:NIPS15,VollmerZT:arxiv15}, we use the {\em mean square error} (MSE) under a target posterior distribution to measure the estimation accuracy for a Bayesian model. Specifically, our utility goal is to evaluate the {\em posterior average} of a test function $\phi(\thetab)$, defined as $\bar{\phi} \triangleq \int \phi(\thetab) p(\thetab|\mathcal{D}) \mathrm{d}\thetab$, with a posterior distribution $p(\thetab|\mathcal{D})$. The posterior average is typically infeasible to compute, thus we use the {\em sample average}, 
$\hat{\phi}_L \triangleq \frac{1}{\sum_t \eta_t} \sum_{t = 1}^T \eta_t \phi(\thetab_{t})$, to approximate $\bar{\phi}$, where $\{\thetab_{t}\}_{t=1}^T$ are the samples from an SG-MCMC algorithm. The MSE we desire is defined as $\mathbb{E}\left(\hat{\phi}_T - \bar{\phi}\right)^2$. 

Our result is summarized in Proposition~\ref{lem:mse}, an extension of Theorem~3 in \citep{ChenWZSC:arxiv17} for the differentially-privacy SG-MCMC with decreasing step sizes. In this section we impose the same assumptions on an SG-MCMC algorithm as in previous work \citep{VollmerZT:arxiv15,ChenDC:NIPS15}, which are detailed in Section~\ref{app:assumption} of the SM. We assume both the corresponding It\^{o} diffusion (in terms of its coefficients) and the numerical method of an SG-MCMC algorithm to be well behaved.

\begin{proposition}\label{lem:mse}
	Under Assumption~\ref{ass:assumption1} in the SM, the MSE of SGLD with a decreasing step size sequence $\{\eta_t<\frac{\epsilon^2Nt^{-1/3}}{c_2^2L^2T^{2/3}\log(1/\delta)}\}$ as in Theorem~\ref{theo:dp} is bounded, for a constant $C$ independent of $\{\eta, T, \tau\}$ and a constant $\Gamma_M$ depending on $T$ and $U(\cdot)$, as $\mathbb{E}\left(\hat{\phi}_L - \bar{\phi}\right)^2$
	\begin{align*} 
		\leq C\left(\frac{2}{3}\left(\frac{N}{n}-1\right)N^2\Gamma_MT^{-1}+\frac{1}{3\tilde{\eta}_0}+2\tilde{\eta}_0^2T^{-2/3}\right)~.
		\end{align*}
	where $\tilde{\eta}_0\triangleq\frac{\epsilon^2}{c_2^2L^2\log(1/\delta)}.$
\end{proposition}

The bound in Proposition~\ref{lem:mse} indicates how the MSE decreases to zero w.r.t.\! the number of iterations $T$ and other parameters. It is consistent with standard SG-MCMC, leading to a similar convergence rate. Interestingly, we can also derive the optimal bounds w.r.t.\! the privacy parameters. For example, the optimal value for $\tilde{\eta}_0$ when fixing other parameters can be seen as $\tilde{\eta}_0 = O\left(T^{2/9}\right)$. Consequently, we have $\epsilon^2 = O\left(L^2T^{2/9}\log(1 / \delta)\right)$ in the optimal MSE setting. Different from the bound of standard SG-MCMC \cite{ChenDC:NIPS15}, when considering a $(\epsilon, \delta)$-DP setting, the MSE bound induces an asymptotic bias term of $\frac{1}{3\tilde{\eta}_0}$ as long as $\epsilon$ and $\delta$ are not equal to zero. 

We also wish to study the MSE under the fixed-step-size case. Consider a general situation, {\it i.e.}, $\eta_t = \eta$, for which \cite{ChenWZSC:arxiv17} has proved the following MSE bound for a fixed steps size, rephrased in Lemma~\ref{lem:mse1}.

\begin{lemma}\label{lem:mse1}
	With the same Assumption as Proposition~\ref{lem:mse}, the MSE of SGLD is bounded as\footnote{With a slight abuse of notation, the constant $C$ is independent of $\{\eta, T, \tau\}$, but might be different from that in Proposition~\ref{lem:mse}.}:
	{\begin{align*}
		\mathbb{E}&\left(\hat{\phi}_L - \bar{\phi}\right)^2 
		\leq C\left(\frac{(\frac{N}{\tau}-1)N^2\Gamma_M}{T} + \frac{1}{T\eta} + \eta^2\right)~.
		\end{align*}}
	Furthermore, the optimal MSE w.r.t.\! the step size $\eta$ is bounded by
	\begin{align*}
		\mathbb{E}&\left(\hat{\phi}_L - \bar{\phi}\right)^2 
		\leq C\left(\frac{(\frac{N}{\tau}-1)N^2\Gamma_M}{T} + T^{-2/3}\right)~,
	\end{align*}
	with the optimal step size being $\eta = O(T^{-1/3})$.
\end{lemma}

From Lemma~\ref{lem:mse1}, the optimal step size, {\it i.e.}, $\eta = O(T^{-1/3})$, is of a lower order than both our  differential-privacy-based algorithm ($\eta = O(T^{-1})$) and the algorithm in \cite{wang2015privacy}, {\it i.e.}, $\eta = O(T^{-1}\log^{-1} T)$. This means that for $T$ large enough, both ours and the method in \cite{wang2015privacy} might not run on the optimal step size setting. A remedy for this is to increase the step size at the cost of increasing privacy loss. Because for the same privacy loss, our step sizes are typically larger than in \cite{wang2015privacy}, our algorithm is able to obtain both higher approximate accuracy and differential privacy. Specifically, to guarantee the desired differential-privacy property as stated in Theorem~\ref{remark:fix_DP}, we substitute a step size of $\eta = \frac{\epsilon^2N}{c^2L^2Tlog(1/\delta)}$ into the MSE formula in Lemma~\ref{lem:mse1}. Consequently, the MSE is bounded by $\mathbb{E}\left(\hat{\phi}_L - \bar{\phi}\right)^2 
\leq C\left(\frac{(\frac{N}{\tau}-1)N^2\Gamma_M}{T} + \frac{c_2^2L^2\log \frac{1}{\delta}}{\epsilon^2 N} + \frac{\epsilon^4 N^2}{c_2^4L^4T^2\log^2(1/\delta)}\right)$, which is smaller than for the method in \cite{wang2015privacy}.

\section{Experiments}
\label{sec:experiment}

\begin{figure}
	\begin{center}
		\includegraphics[width=0.5\textwidth]{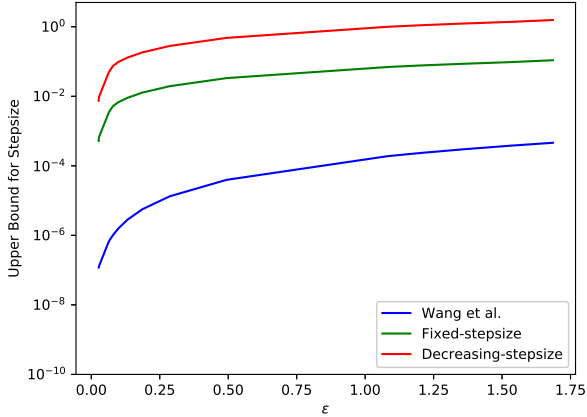}
		\caption{Upper bounds for fixed-step size and decreasing-step size with different privacy loss $\epsilon$, as well as the upper bound from \citet{wang2015privacy}.} 
		\label{fig:bound}
	\end{center}
\end{figure}
We test the proposed differentially-private SG-MCMC algorithms by considering several tasks, including logistic regression and deep neural networks, and compare with related Bayesian and optimization methods in terms of both algorithm privacy and utility. 

\subsection{Upper Bound}\label{sec:exp_upper}
\vspace{-0.1cm}
We first compare our upper bound for the step size in Section~\ref{sec:step size_bound} with the bound of \citet{wang2015privacy}. Note this upper bound denotes the largest step size allowed to preserve $(\epsilon,\delta)$-DP.

In this simulation experiment, we use the following setting: $N=50,000$, $T=10,000$, $L=1$, and $\delta=10^{-5}$. We vary $\epsilon$ from $0.02$ to $1.7$ for different differential-privacy settings, for both ours (fixed and decreasing-step size cases) and the bound in \cite{wang2015privacy}, with results in Figure~\ref{fig:bound}. It is clear that our bounds give much larger step sizes than from \cite{wang2015privacy} at a same privacy loss, {\it e.g.}, $10^{-1}$ vs.\! $10^{-4}$. Our step sizes appear to be much more practical in real applications.

In the rest of our experiments, we focus on using the decreasing-step size SGLD as it gives a nicer MSE bound as shown in Proposition~\ref{lem:mse}.
For the parameters in our bounds, {\it i.e.}, $(N,T,\epsilon,\delta, L)$, the default setting is often chosen to be $\delta=O(1/N)$ and $T=O(N)$; $L$ is typically selected from a range such as $L\in\{0.1, 1, 10\}$. In this experiment, we investigate the sensitivity of our proposed upper bound w.r.t.\! $N$ and $L$ when fixing other parameters. The results are plotted in Figure~\ref{fig:diffbound}, from which we observe that our proposed step size bound is stable in terms of the data size $N$, and is approximately proportional to $1/L$. Such a conclusion is not a direct implication from the upper bound formula in Theorem~\ref{theo:dp}, as the constant $c_2$ also depends on $(N,T,\epsilon,\delta,L)$.

\begin{figure}
	\begin{center}
		\includegraphics[width=0.5\textwidth]{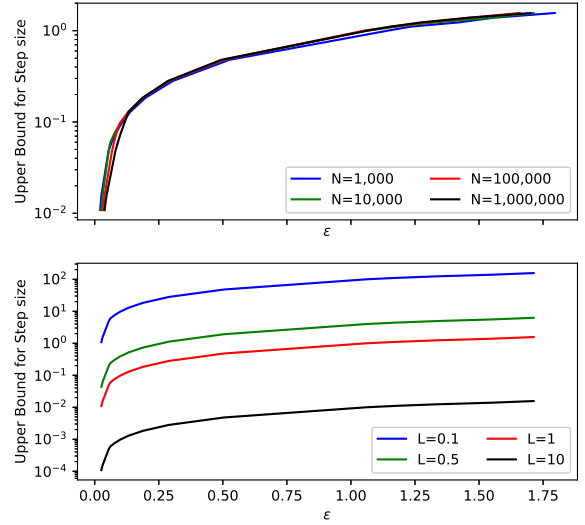}
		\caption{Step size upper bounds for $N=10^{3}, 10^{4}, 10^{5}, 10^{6}$ with fixed $L=1$ (top), and $L=0.1, 0.5, 1.0, 10.0$ with fixed $N=10^{4}$ (bottom). In both simulations, we let $\delta=1/N$ and $T=N$.}
		\label{fig:diffbound}
	\end{center}
\end{figure}

The result also indicates a rule for choosing step sizes in practice by using our upper bound, which fall into the range of $(10^{-4}, 0.1)$. When using such step sizes, we observe that the standard SGLD automatically preserves $(\epsilon,\delta)$-DP even when $\epsilon$ is small.

\subsection{Logistic Regression}

In the remaining experiments, we compare our proposed differentially-private SGLD (DP-SGLD) with other methods.
The Private Aggregation of Teacher Ensembles (PATE) model proposed in \citet{papernot2016semi} is the state-of-the-art framework for differentially private training of machine learning models. PATE takes advantage of the moment accountant method for privacy loss calculation, and uses a knowledge-transfer technique via semi-supervised learning, to build a teacher-student-based model. This framework first trains multiple teachers with private data; these teachers then differentially and privately release aggregated knowledge, such as label assignments on several public data points, to multiple students. The students then use the released knowledge to train their models in a supervised learning setting, or they can incorporate unlabeled data in a semi-supervised learning setting. The semi-supervised setting generally works for many machine learning models, yet it requires a large amount of non-private unlabeled data for training, which are not always available in practice. Thus, we did not consider this setting in our experiments.

We compare DP-SGLD with PATE and the Hybrid Posterior Sampling algorithm on the Adult data set from the UCI Machine Learning Repository \citep{Lichman:2013}, for a binary classification task with Bayesian logistic regression, under the DP setting. We fix $\delta=10^{-4}$, and compare the classification accuracy while varying $\epsilon$. We repeat each experiment ten times, and report averages and the standard deviations, as illustrated in Figure \ref{fig:glm}.
 
\begin{figure}
	\begin{center}
		\includegraphics[width=0.5\textwidth]{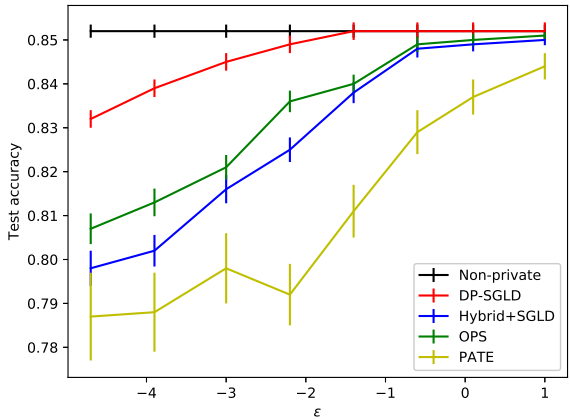}
		\caption{Test accuracies on a classification task based on Bayesian logistic regression for One-Posterior Sample (OPS), Hybrid Posterior sampling based on SGLD, and our proposed DP-SGLD with different choice of privacy loss $\epsilon$. The non-private baseline is obtained by standard SGLD.} 
		\label{fig:glm}
	\end{center}
\end{figure}

Our proposed DP-SGLD achieves a higher accuracy compared to other methods and is close to the baseline where the plain SGLD is used. In fact, when $\epsilon\approx 0.3$ or above, our DP-SGLD becomes the standard SGLD, therefore has the same test accuracy as the baseline. 
Note that PATE obtains the worst performance in this experiment. This might be because when $\epsilon$ is small and without unlabeled data, the students in this framework are restricted to using supervised learning with an extremely small amount of training data.

\subsection{Deep Neural Networks} 

We test our methods for training deep neural networks under differentially-private settings. We compare our methods with PATE and the DP-SGD proposed in \citet{abadi2016deep}. Since the performance of PATE highly depends on the availability of public unlabeled data, we allow it to access a certain amount of unlabeled data, though it is not a fair comparison to our method. We do not include the results with Hybrid Posterior sampling, as it does not converge due to its small step sizes in the experiments.

We use two datasets: ($i$) the standard MNIST dataset
for handwritten digit recognition, consisting of 60,000 training examples and 10,000 testing examples \citep{lecun-mnisthandwrittendigit-2010}; and ($ii$) the Street View House Number (SVHN) dataset, which contains 600,000
$32\times32$ RGB images of printed digits obtained
from pictures of house number in street view \citep{netzer2011reading}. We use the same network structure as for the PATE model, which contains two stacked convolutional layers and one fully connected layer with ReLUs for MNIST, and two more convolutional layers for SVHN. We use standard Gaussian priors for the weights of the DNN. For the MNIST dataset, the standard SGLD with step size $\eta_t=0.1t^{-1/3}$ satisfies $(\epsilon,\delta)$-DP for $\epsilon=0.10$ and $\delta=10^{-5}$ when we set $L=4$. For the SVHN dataset, the standard SGLD with step size $\eta_t=0.1t^{1/3}$ satisfies $(\epsilon,\delta)$-DP for $\epsilon=0.12$ and $\delta=10^{-6}$ when we set $L=1$. In both settings, we let $q=1/\sqrt{N}$ to satisfy the second condition in Theorem \ref{theo:dp}. In addition, we also ran a differentially-private version of the SGHMC for comparison. The test accuracies are shown in Table~\ref{table: nn}. 
It is shown that SGLD and SGHMC obtain better test accuracy than the state-of-the-art differential privacy methods, remarkably with much less privacy loss. They even outperformed the non-private baseline model using Adam, due to the advantages of Bayesian modeling.

\begin{table}[ht]
	\label{table: nn}
	\centering{
	\caption{Test accuracies on MNIST and and SVHN for different methods.}
	\begin{tabular}{ c |l| c|c| c }
		Dataset &Methods &$\epsilon$ &$\delta$ & Accuracy\\\hline
	  &Non-Private &  &  & 99.23\%\\\cline{2-5}
	  &PATE(100) &$2.04$&$10^{-5}$ & 98.00\% \\\cline{2-5}
	  MNIST&PATE(1000) &$8.03$&$10^{-5}$ & 98.10\% \\\cline{2-5}
	  &\textbf{DP-SGLD} &$0.10$&$10^{-5}$& 99.12\% \\\cline{2-5}
	  &\textbf{DP-SGHMC} &$0.24$&$10^{-5}$&$\mathbf{99.28}$\%\\ 
	  \hline

	  &Non-Private &  &  & 92.80\%\\\cline{2-5}
	  &PATE(100) &$5.04$&$10^{-6}$ & 82.76\% \\\cline{2-5}
	  SVHN&PATE(1000) &$8.19$&$10^{-6} $ & 90.66\% \\\cline{2-5}
	  &\textbf{DP-SGLD} &$0.12$&$10^{-6}$& 92.14\% \\\cline{2-5}
	  &\textbf{DP-SGHMC} &$0.43$&$10^{-6}$&$\mathbf{92.84}$\%\\
	  \hline
	\end{tabular}
	}
\end{table}
 
\section{Related Work}
\label{sec:relatedwork}
There have been several papers that have considered differentially-private stochastic gradient based methods. For example, \citet{song2013stochastic} proposed a differentially-private stochastic gradient descent (SGD) algorithm, which requires a large amount of noise when mini-batches are randomly sampled. The theoretical performance of noisy SGD is studied in \citet{bassily2014differentially} for the special case of convex loss functions. Therefore, for a non-convex loss function, a common setting for many machine learning models, there are no theoretical guarantee on performance. In \citet{abadi2016deep} another differentially private SGD was proposed, requiring a smaller variance for added Gaussian noise, yet it still did not provide theoretical guarantees on utility. On the other hand, the standard SG-MCMC has been shown to be able to converge to the target posterior distribution in theory. In this paper, we discuss the effect of our modification for differential privacy on the performance of the SG-MCMC, which endows theoretical guarantees on the bounds for the mean squared error of the posterior mean. 

Bayesian modeling provides an effective framework for privacy-preserving data analysis, as posterior sampling naturally introduces noise into the system, leading to differential privacy \citep{dimitrakakis2014robust,wang2015privacy}. In \citet{foulds2016theory}, the privacy for sampling from exponential families with a Gibbs sampler was studied. In \citet{wang2015privacy} a comprehensive analysis was proposed on the differential privacy of SG-MCMC methods. As a comparison, we have derived a tighter bound for the amount of noise required to guarantee a certain differential privacy, yielding a more practical upper bound for the step size.

\section{Conclusion}
\label{sec:conclusion}
Previous work on differential privacy has modified existing algorithms, or has built complicated
frameworks that sacrifice a certain amount of performance for privacy. In some cases the privacy loss may be relatively large. This paper has addressed a privacy analysis for SG-MCMC, a standard class of methods for scalable posterior sampling for Bayesian models. We have significantly relaxed the condition for SG-MCMC methods being differentially private, compared to previous works. Our results indicate that standard SG-MCMC methods have strong privacy guarantees for problems in large scale. In addition, we have proposed theoretical analysis on the estimation performance of differentially private SG-MCMC methods. Our results show that even when there is a strong privacy constraint, the differentially private SG-MCMC still endows a guarantee on the model performance. Our experiments have shown that with our analysis, the standard SG-MCMC methods achieve both state-of-the-art utility and strong privacy compared with related methods on multiple tasks, such as logistic regression and deep neural networks. 

Our results also shed lights onto how SG-MCMC methods help improving the generalization for training models, as it is well acknowledged that there is a connection between differential privacy and generalization for a model (cite Learning with Differential Privacy: Stability, Learnability and the Sufficiency and Necessity of ERM Principle). For example, in \cite{SaatchiW:NIPS17}, a Bayesian GAN model trained with SGHMC is proposed and shows promising performance in avoiding mode collapse problem in GAN training. According to \cite{AroraGLMZ:ICML17}, the mode collapse problem is potentially caused by weak generalization. Therefore, it is very likely that Bayesian GAN moderated mode collapse problem because SGHMC naturally leads to better generalization.

\bibliography{references.bib}

\begin{thebibliography}{28}
\providecommand{\natexlab}[1]{#1}
\providecommand{\url}[1]{\texttt{#1}}
\expandafter\ifx\csname urlstyle\endcsname\relax
  \providecommand{\doi}[1]{doi: #1}\else
  \providecommand{\doi}{doi: \begingroup \urlstyle{rm}\Url}\fi

\bibitem[Abadi et~al.(2016)Abadi, Chu, Goodfellow, McMahan, Mironov, Talwar,
  and Zhang]{abadi2016deep}
Mart{\'\i}n Abadi, Andy Chu, Ian Goodfellow, H~Brendan McMahan, Ilya Mironov,
  Kunal Talwar, and Li~Zhang.
\newblock Deep learning with differential privacy.
\newblock In \emph{Proceedings of the 2016 ACM SIGSAC Conference on Computer
  and Communications Security}, pages 308--318. ACM, 2016.

\bibitem[Arora et~al.(2017)Arora, Ge, Liang, Ma, and Zhang]{AroraGLMZ:ICML17}
S.~Arora, R.~Ge, Y.~Liang, T.~Ma, and Y.~Zhang.
\newblock Generalization and equilibrium in generative adversarial nets
  ({GANs}).
\newblock In \emph{ICML}, 2017.

\bibitem[Bassily et~al.(2014)Bassily, Smith, and
  Thakurta]{bassily2014differentially}
Raef Bassily, Adam Smith, and Abhradeep Thakurta.
\newblock Differentially private empirical risk minimization: Efficient
  algorithms and tight error bounds.
\newblock \emph{arXiv preprint arXiv:1405.7085}, 2014.

\bibitem[Chen et~al.(2015)Chen, Ding, and Carin]{ChenDC:NIPS15}
C.~Chen, N.~Ding, and L.~Carin.
\newblock On the convergence of stochastic gradient {MCMC} algorithms with
  high-order integrators.
\newblock In \emph{NIPS}, 2015.

\bibitem[Chen et~al.(2017)Chen, Wang, Zhang, Su, and Carin]{ChenWZSC:arxiv17}
C.~Chen, W.~Wang, Y.~Zhang, Q.~Su, and L.~Carin.
\newblock A convergence analysis for a class of practical variance-reduction
  stochastic gradient mcmc.
\newblock \penalty0 (arXiv:1709.01180), 2017.
\newblock URL \url{https://arxiv.org/abs/1709.01180}.

\bibitem[Chen et~al.(2014)Chen, Fox, and Guestrin]{pmlr-v32-cheni14}
Tianqi Chen, Emily Fox, and Carlos Guestrin.
\newblock Stochastic gradient hamiltonian monte carlo.
\newblock In Eric~P. Xing and Tony Jebara, editors, \emph{Proceedings of the
  31st International Conference on Machine Learning}, volume~32 of
  \emph{Proceedings of Machine Learning Research}, pages 1683--1691, Bejing,
  China, 22--24 Jun 2014. PMLR.

\bibitem[Dimitrakakis et~al.(2014)Dimitrakakis, Nelson, Mitrokotsa, and
  Rubinstein]{dimitrakakis2014robust}
Christos Dimitrakakis, Blaine Nelson, Aikaterini Mitrokotsa, and Benjamin~IP
  Rubinstein.
\newblock Robust and private bayesian inference.
\newblock In \emph{International Conference on Algorithmic Learning Theory},
  pages 291--305. Springer, 2014.

\bibitem[Ding et~al.(2014)Ding, Fang, Babbush, Chen, Skeel, and
  Neven]{DingFBCSN:NIPS14}
N.~Ding, Y.~Fang, R.~Babbush, C.~Chen, R.~D. Skeel, and H.~Neven.
\newblock Bayesian sampling using stochastic gradient thermostats.
\newblock In \emph{NIPS}, 2014.

\bibitem[Dwork(2008)]{dwork2008differential}
Cynthia Dwork.
\newblock Differential privacy: A survey of results.
\newblock In \emph{International Conference on Theory and Applications of
  Models of Computation}, pages 1--19. Springer, 2008.

\bibitem[Dwork et~al.(2006)Dwork, McSherry, Nissim, and
  Smith]{dwork2006calibrating}
Cynthia Dwork, Frank McSherry, Kobbi Nissim, and Adam Smith.
\newblock Calibrating noise to sensitivity in private data analysis.
\newblock Springer, 2006.

\bibitem[Dwork et~al.(2014)Dwork, Roth, et~al.]{dwork2014algorithmic}
Cynthia Dwork, Aaron Roth, et~al.
\newblock The algorithmic foundations of differential privacy.
\newblock \emph{Foundations and Trends{\textregistered} in Theoretical Computer
  Science}, 9\penalty0 (3--4):\penalty0 211--407, 2014.

\bibitem[Foulds et~al.(2016)Foulds, Geumlek, Welling, and
  Chaudhuri]{foulds2016theory}
James Foulds, Joseph Geumlek, Max Welling, and Kamalika Chaudhuri.
\newblock On the theory and practice of privacy-preserving bayesian data
  analysis.
\newblock \emph{arXiv preprint arXiv:1603.07294}, 2016.

\bibitem[Ghosh(2011)]{Ghosh:book11}
A.~P. Ghosh.
\newblock \emph{Backward and Forward Equations for Diffusion Processes}.
\newblock Wiley Encyclopedia of Operations Research and Management Science,
  2011.

\bibitem[LeCun and Cortes(2010)]{lecun-mnisthandwrittendigit-2010}
Yann LeCun and Corinna Cortes.
\newblock {MNIST} handwritten digit database.
\newblock 2010.
\newblock URL \url{http://yann.lecun.com/exdb/mnist/}.

\bibitem[Li et~al.(2016)Li, Chen, Carlson, and Carin]{LICCC:AAAI16}
C.~Li, C.~Chen, D.~Carlson, and L.~Carin.
\newblock Preconditioned stochastic gradient {L}angevin dynamics for deep
  neural networks.
\newblock In \emph{AAAI}, 2016.

\bibitem[Lichman(2013)]{Lichman:2013}
M.~Lichman.
\newblock {UCI} machine learning repository, 2013.
\newblock URL \url{http://archive.ics.uci.edu/ml}.

\bibitem[Ma et~al.(2015)Ma, Chen, and Fox]{MaCF:NIPS15}
Y.~A. Ma, T.~Chen, and E.~B. Fox.
\newblock A complete recipe for stochastic gradient {MCMC}.
\newblock In \emph{NIPS}, 2015.

\bibitem[Mattingly et~al.(2010)Mattingly, Stuart, and
  Tretyakov]{MattinglyST:JNA10}
J.~C. Mattingly, A.~M. Stuart, and M.~V. Tretyakov.
\newblock Construction of numerical time-average and stationary measures via
  {P}oisson equations.
\newblock \emph{SIAM J. NUMER. ANAL.}, 48\penalty0 (2):\penalty0 552--577,
  2010.

\bibitem[Netzer et~al.()Netzer, Wang, Coates, Bissacco, Wu, and
  Ng]{netzer2011reading}
Yuval Netzer, Tao Wang, Adam Coates, Alessandro Bissacco, Bo~Wu, and Andrew~Y
  Ng.
\newblock Reading digits in natural images with unsupervised feature learning.

\bibitem[Papernot et~al.(2016)Papernot, Abadi, Erlingsson, Goodfellow, and
  Talwar]{papernot2016semi}
Nicolas Papernot, Mart{\'\i}n Abadi, {\'U}lfar Erlingsson, Ian Goodfellow, and
  Kunal Talwar.
\newblock Semi-supervised knowledge transfer for deep learning from private
  training data.
\newblock \emph{arXiv preprint arXiv:1610.05755}, 2016.

\bibitem[Pascanu et~al.(2013)Pascanu, Mikolov, and
  Bengio]{pascanu2013difficulty}
Razvan Pascanu, Tomas Mikolov, and Yoshua Bengio.
\newblock On the difficulty of training recurrent neural networks.
\newblock In \emph{International Conference on Machine Learning}, pages
  1310--1318, 2013.

\bibitem[Patterson and Teh(2013)]{PattersonT:NIPS13}
S.~Patterson and Y.~W. Teh.
\newblock Stochastic gradient {R}iemannian {L}angevin dynamics on the
  probability simplex.
\newblock In \emph{NIPS}, 2013.

\bibitem[Saatchi and Wilson(2017)]{SaatchiW:NIPS17}
Y.~Saatchi and A.~G. Wilson.
\newblock Bayesian {GAN}.
\newblock In \emph{NIPS}, 2017.

\bibitem[Shokri et~al.(2017)Shokri, Stronati, Song, and
  Shmatikov]{shokri2017membership}
Reza Shokri, Marco Stronati, Congzheng Song, and Vitaly Shmatikov.
\newblock Membership inference attacks against machine learning models.
\newblock In \emph{Security and Privacy (SP), 2017 IEEE Symposium on}, pages
  3--18. IEEE, 2017.

\bibitem[Song et~al.(2013)Song, Chaudhuri, and Sarwate]{song2013stochastic}
Shuang Song, Kamalika Chaudhuri, and Anand~D Sarwate.
\newblock Stochastic gradient descent with differentially private updates.
\newblock In \emph{Global Conference on Signal and Information Processing
  (GlobalSIP), 2013 IEEE}, pages 245--248. IEEE, 2013.

\bibitem[Teh et~al.(2016)Teh, Thiery, and Vollmer]{TehTV:arxiv14}
Y.~W. Teh, A.~H. Thiery, and S.~J. Vollmer.
\newblock Consistency and fluctuations for stochastic gradient {L}angevin
  dynamics.
\newblock \emph{JMLR}, \penalty0 (17):\penalty0 1--33, 2016.

\bibitem[Vollmer et~al.(2016)Vollmer, Zygalakis, and Teh]{VollmerZT:arxiv15}
S.~J. Vollmer, K.~C. Zygalakis, and Y.~W. Teh.
\newblock Exploration of the {(Non-)A}symptotic bias and variance of stochastic
  gradient {L}angevin dynamics.
\newblock \emph{JMLR}, 2016.

\bibitem[Wang et~al.(2015)Wang, Fienberg, and Smola]{wang2015privacy}
Yu-Xiang Wang, Stephen Fienberg, and Alex Smola.
\newblock Privacy for free: Posterior sampling and stochastic gradient monte
  carlo.
\newblock In \emph{Proceedings of the 32nd International Conference on Machine
  Learning (ICML-15)}, pages 2493--2502, 2015.

\end{thebibliography}
\bibliographystyle{plainnat}

\newpage

\appendix

\section{Proof of Theorem 3}\label{app:theorem3}

We first prove Algorithm \ref{alg:1} is $(\epsilon,\delta)$-DP if we change the variance of $\mathbf{z}_t$ to be $\sigma^2_t=\frac{c_2^2L^2T^{2/3}t^{1/3}\log(1/\delta)}{\epsilon^2N^2}\eta_t^2I$ for some constant $c_2$.

It is easy to see that SGLD in Algorithm~\ref{alg:1} consists of a sequence of updates for the model parameter $\thetab$. Each update corresponds to a random mechanism $\mathcal{M}_i$ defined in Theorem~\ref{theorem:1}, thus we will first derive the moments accountant for each iteration. In each iteration, the only data access is $\sum_{i\in J_t}\tilde{g}_t(\db_i)$ in Step 6. Therefore, in the following, we only focus on the interaction between $\sum_{i\in J_t}\tilde{g}_t(\db_i)$ and the noise $\mathbf{z}_t$, which is essentially\footnote{In this paper, we only consider the case for which we choose priors that do not depend on the data, as is common in the Bayesian setting.} $\frac{\eta_t}{\tau}\sum_{i\in J_t}\bar{g}_t(\db_i) +\mathbf{z}_t$.

To simplify the notation, we let $\tilde{\eta}^2=\frac{\sigma^2_t\tau^2}{L^2\eta_t^2t^{1/3}}$, and the variance of $\mathbf{z}_t$ can be rewritten as $\sigma_t^2=(\tilde{\eta}^2L^2\eta_t^2t^{1/3}/\tau^2)I$\footnote{Later we will show the optimal decreasing ratio for the step size is $t^{1/3}$.}. Then we have:
 {\small\begin{align*}
 \frac{\eta_t}{\tau}\sum_{i\in J_t}\bar{g}_t(\db_i) +\mathbf{z}_t
 &=\frac{\eta_t}{\tau}\left(\sum_{i\in J_t}\bar{g}_t(\db_i) +N(0, (\sigma_t^2\tau^2/\eta^2_t)I)\right)\\
 &=\frac{\eta_tL}{\tau}\left(\frac{1}{L}\sum_{i\in J_t}\bar{g}_t(\db_i) +N(0, \tilde{\eta}^2t^{1/3}I)\right) 
 \end{align*}}

If we let $f(\mathbf{d}_i)=\frac{1}{L}\hat{g}_t(\mathbf{d}_i)$ and $\sigma^2=\tilde{\eta}^2t^{1/3}$, we can apply Lemma \ref{lemma:1} to calculate the upper bound for the log moment of the privacy loss random variable for the $t^{\text{th}}$ iteration to be $$\alpha(\lambda)\leq t^{-1/3}q^2\lambda^2/\tilde{\eta}^2$$
as long as the conditions in Lemma \ref{lemma:1} are satisfied, that is $\tilde{\eta}^2t^{1/3}\geq 1$ and the mini-batch sampling probability $q<\frac{1}{16\tilde{\eta}t^{1/6}}$.

Using the composability property of the moments accountant in Theorem~\ref{theorem:1}, over $T$ iterations, the log moment of the privacy loss random variable is bounded by 
\begin{align*}
 \alpha(\lambda)\leq \sum_{t=1}^T(t^{-1/3})q^2\lambda^2/\tilde{\eta}^2~.
\end{align*}

According to the tail bound property in Theorem \ref{theorem:1}, $\delta$ is the minimum of $\exp\left(\alpha_{\mathcal{M}}(\lambda) - \lambda\epsilon\right)$ w.r.t.\! $\lambda$. However, since $\lambda$ is an integer, a closed form for this minimum is generally intractable. Nevertheless, to guarantee $(\epsilon,\delta)$-DP, it suffices that

\begin{align}\label{eq:condition1}
 \sum_{t=1}^T(t^{-1/3})q^2\lambda^2/ \tilde{\eta}^2\leq \lambda\epsilon/2,\ \ \ \exp(-\lambda\epsilon/2)\leq \delta~,
\end{align}

We also require that our choice of parameters satisfies Lemma \ref{lemma:1}. Consequently, we have 
\begin{align}\label{eq:condition2}
 \lambda\leq\tilde{\eta}^2t^{1/3}\log(1/q\tilde{\eta}^2t^{1/3})\leq \tilde{\eta}^2\log(1/q\tilde{\eta}^2)
\end{align} 
Since $\sum_{t=1}^T t^{-1/3}=O(T^{2/3})$, we can use a similar technique\footnote{Further explained in Section~\ref{app:constants} of the SM.} as in \citet{abadi2016deep} to find explicit constants $c_1$ and $c_2$ such that when $\epsilon=c_1q^2T^{2/3}$ and $\tilde{\eta}=c_2\frac{q\sqrt{T^{2/3}\log(1/\delta)}}{\epsilon}$, the conditions \eqref{eq:condition1} \eqref{eq:condition2} are satisfied. If we plug in $\tilde{\eta}$ and $q$, we have proved that Algorithm \ref{alg:1} is $(\epsilon,\delta)$-DP when $\mathbf{z}_i\sim N(0,\frac{c_2^2L^2T^{2/3}t^{1/3}\log(1/\delta)}{\epsilon^2N^2}\eta_t^2I)$.
 
For the second step of the proof, we prove that Algorithm~\ref{alg:1} is $(\epsilon, \delta)$-DP when the original variance of $\mathbf{z}_t$ is used, {\it i.e.}, $\sigma_t = \frac{\eta_t}{N}$. This is straightforward because when $\eta_t<\frac{\epsilon^2Nt^{-1/3}}{c_2^2L^2T^{2/3}\log(1/\delta)}$ we have $\frac{c_2^2L^2T^{2/3}t^{1/3}\log(1/\delta)}{\epsilon^2N^2}\eta_t^2<\eta_t/N$ as long as the step size $\eta_t$ is positive. Adding more noise decreases the privacy loss. To satisfy $(\epsilon,\delta)$-DP, it suffices to set the variance of $\mathbf{z}_i$ as $\eta_t/N$, which gives the original Algorithm \ref{alg:1}, a variant of the standard SGLD algorithm with decreasing step size.

In addition, Lemma \ref{lemma:1} requires $\tilde{\eta}^2t^{1/3}\geq 1$ and $q<\frac{1}{16\tilde{\eta}t^{1/6}}$. Note $\eta_t=\frac{N}{t^{1/3}\tilde{\eta}^2L^2}$, we also need to ensure $\eta_t\leq \frac{N}{L^2}$ and $\eta_t>\frac{q^2N}{256L^2}$.

\section{Assumptions on SG-MCMC Algorithms}\label{app:assumption}

For the diffusion in \eqref{eq:ito}, we first define the generator $\mathcal{L}$ as:
\begin{align}
\mathcal{L} \psi \triangleq \frac{1}{2}\nabla\psi \cdot F + \frac{1}{2}g(\thetab) g(\thetab)^{*}:D^2 \psi~,
\end{align}
where $\psi$ is a measurable function, $D^k \psi$ means the $k$-derivative of $\psi$, $*$ means transpose. $\ab\cdot\bb \triangleq \ab^T\bb$ for two vectors $\ab$ and $\bb$, $\Ab:\Bb \triangleq \mbox{trace}(\Ab^T\Bb)$ for two matrices $\Ab$ and $\Bb$. Under certain assumptions, there exists a function, $\phi$, such that the following Poisson equation is satisfied \cite{MattinglyST:JNA10}:
\begin{align}\label{eq:poissoneq}
\mathcal{L} \psi = \phi - \bar{\phi}~,
\end{align}
where $\bar{\phi} \triangleq \int \phi(\thetab) \rho(\mathrm{d} \thetab)$ denotes the model average, with $\rho$ being the equilibrium distribution for the diffusion \eqref{eq:ito}, which is assumed to coincide with the posterior distribution $p(\thetab|\mathcal{D})$. The following assumptions are made for the SG-MCMC algorithms \citep{VollmerZT:arxiv15,ChenDC:NIPS15}.

\begin{assumption}\label{ass:assumption1}
	The diffusion \eqref{eq:ito} is ergodic. Furthermore, the solution of \eqref{eq:poissoneq} exists, and the solution functional $\psi$ satisfies the following properties:
	\begin{itemize}
		\item $\psi$ and its up to 3th-order derivatives $\mathcal{D}^k \psi$, are bounded by a function $\mathcal{V}$, {\it i.e.}, $\|\mathcal{D}^k \psi\| \leq C_k\mathcal{V}^{p_k}$ for $k=(0, 1, 2, 3)$, $C_k, p_k > 0$.
		\item The expectation of $\mathcal{V}$ on $\{\xb_{l}\}$ is bounded: $\sup_l \mathbb{E}\mathcal{V}^p(\xb_{l}) < \infty$.
		\item $\mathcal{V}$ is smooth such that $\sup_{s \in (0, 1)} \mathcal{V}^p\left(s\xb + \left(1-s\right)\yb\right) \leq C\left(\mathcal{V}^p\left(\xb\right) + \mathcal{V}^p\left(\yb\right)\right)$, $\forall \xb \in \mathbb{R}^m, \yb \in \mathbb{R}^m, p \leq \max\{2p_k\}$ for some $C > 0$.
	\end{itemize}
\end{assumption} 

\section{Calculating Constants in Moment Accountant Methods}\label{app:constants}
For calculating the constants $c_1$ and $c_2$, which is a part of the moment accoutant method, we refer to \url{https://github.com/tensorflow/models} \footnote{This is under the Apache License, Version 2.0} as an implimentation of the moment accountant method. A comprehensive description for the implimentation can be found int \citet{abadi2016deep}. 

This code allows one to calculate the corresponding $\epsilon(\delta)$ given $\delta(\epsilon),q,T,\eta_0$ by enumerating all the possible integers under a certain threhosld as the candidate value of $\lambda$ and selecting the one that minimizes $\epsilon(\delta)$. Once $\epsilon(\delta)$ is determined, it is easy to calculate $c_1$ and $c_2$ for evaluating the upper bound for the step size.

\section{Proof of Theorem~\ref{remark:fix_DP}}\label{app:fixed_DP}

Claim: Under the same setting as Theorem~\ref{theo:dp}, but using a fixed-step size $\eta_t = \eta$, Algorithm \ref{alg:1} satisfies $(\epsilon,\delta)$-DP whenever $\eta<\frac{\epsilon^2N}{c^2L^2Tlog(1/\delta)}$ for another constant $c$.

\begin{proof}
The only change of the proof for fixed step size is that the expression for the variance of the Gaussian noise $\mathbf{z}_t$ becomes $\sigma^2_t=\eta_0^2L^2\eta_t^2/\tau^2$ for fixed step size. We still apply Theorem \ref{theorem:1} and Lemma \ref{lemma:1} to find the required conditions for $(\epsilon, \delta)$-DP:
$$T q^2\lambda^2/\eta_0^2\leq \lambda\epsilon/2$$
$$\exp(-\lambda\epsilon/2)\leq \delta, \lambda\leq\eta_0^2\log(1/q\eta_0)$$

Using the method described in the previous section, one can find $c_3$ and $c_4$ such that when $\epsilon=c_3q^2T$ and $\eta_0=c_4\frac{q\sqrt{\log(1/\delta)}}{\epsilon}$ satisfy the above conditions. Then if we plug in $\eta_0$ and $q$, and compare it to $\eta/N$, it is easy to see Algorithm \ref{alg:1} satisfies $(\epsilon,\delta)$-DP when $\eta<\frac{\epsilon^2N}{c_4^2L^2T\log(1/\delta)}$.
\end{proof}

\section{Proof of Proposition~\ref{lem:mse}}

	Claim: Under Assumption~\ref{ass:assumption1} in the section~\ref{app:assumption}, the MSE of SGLD with a decreasing step size sequence $\{\eta_t<\frac{\epsilon^2Nt^{-1/3}}{c_2^2L^2T^{2/3}\log(1/\delta)}\}$ as in Theorem~\ref{theo:dp} is bounded, for a constant $C$ independent of $\{\eta, T, \tau\}$ and a constant $\Gamma_M$ depending on $T$ and $U(\cdot)$, as $\mathbb{E}\left(\hat{\phi}_L - \bar{\phi}\right)^2$
	\begin{align*} 
		\leq C\left(\frac{2}{3}\left(\frac{N}{n}-1\right)N^2\Gamma_MT^{-1}+\frac{1}{3\tilde{\eta}_0}+2\tilde{\eta}_0^2T^{-2/3}\right)~.
		\end{align*}
	where $\tilde{\eta}_0\triangleq\frac{\epsilon^2}{c_2^2L^2\log(1/\delta)}.$

\begin{proof}
	
First, we adopt the MSE formula for the decreasing-step-size SG-MCMC with Euler integrator (1-st order integrator) from Theorem~5 of \cite{ChenDC:NIPS15}, which is written as
\begin{align}\label{eq:msebound}
	\mathbb{E}\left(\hat{\phi}_L - \bar{\phi}\right)^2 \leq C\left(\sum_{t=1}^T \frac{\eta_t^2}{S_T^2}\mathbb{E}\left\|\Delta V_t\right\|^2 + \frac{1}{S_T} + \frac{(\sum_{t=1}^T \eta_t^{2})^2}{S_T^2} \right)~,
\end{align}
where $S_T \triangleq \sum_{t = 1}^T\eta_t$, and $\Delta V_t$ is a term related to $\tilde{g}_t$, which, according to Theorem~3 of \cite{ChenWZSC:arxiv17}, can be simplified as
\begin{align}\label{eq:deltav}
	&\mathbb{E}\left|\Delta V_l\right|^2 \nonumber\\
	=&\frac{(N - \tau)N^2}{\tau}\left(\frac{1}{N^2}\sum_{i,j}\mathbb{E}\alphab_{li}\alphab_{lj} - \frac{2}{N(N-1)}\sum_{i\leq j}\mathbb{E}\alphab_{li}\alphab_{lj}\right) \nonumber\\
	\triangleq& \frac{(N - \tau)N^2}{\tau} \Gamma_t~.
\end{align}

Let $\Gamma_M \triangleq \max_t \Gamma_t$. Substituting \eqref{eq:deltav} into \eqref{eq:msebound}, we have

\begin{align}\label{eq:msebound1}
&\mathbb{E}\left(\hat{\phi}_L - \bar{\phi}\right)^2 \leq \\
&C\left(\frac{\sum_t^T \eta_t^2}{\left(\sum_t^T\eta_t\right)^2}\left(\frac{N}{\tau}-1\right)N^2\Gamma_M+\frac{1}{\sum_t^T\eta_t}+\frac{\left(\sum_t^T\eta_t^2\right)^2}{\left(\sum_t^T\eta_t\right)}\right) \nonumber
\end{align}
Now, if we assume $\tilde{\eta}_0=\frac{\epsilon}{c_2^2L^2\log(1/\delta)}$, then we rewrite $\eta_t=\eta_0t^{-1/3}T^{-2/3}$.

Note $\sum_t^T t^p\approx \frac{1}{p+1}T^{p+1}$. Plug this into the bound in \eqref{eq:msebound1}, we have:

\begin{align*}
	&\mathbb{E}\left(\hat{\phi}_L - \bar{\phi}\right)^2 \leq \\
	&C\left(\frac{\sum_t^T \eta_t^2}{\left(\sum_t^T\eta\right)^2}\left(\frac{N}{\tau}-1\right)N^2\Gamma_M+\frac{1}{\sum_t^T\eta_t}+\frac{\left(\sum_t^T\eta_t^2\right)^2}{\left(\sum_t^T\eta_t\right)^2}\right)\\
	\leq&C\left(\frac{2}{3}\left(\frac{N}{\tau}-1\right)N^2\Gamma_MT^{-1}+\frac{1}{3\tilde{\eta}_0}+2\tilde{\eta}_0^2T^{-2/3}\right)
\end{align*}
\end{proof}

\end{document}